\documentclass[10pt,twocolumn,letterpaper]{article}
\usepackage{cvpr}
\usepackage{times}
\usepackage{epsfig}
\usepackage{graphicx}
\usepackage{amsmath}
\usepackage{amssymb}
\usepackage[thmmarks]{ntheorem}
\usepackage{cases}
{
\theoremstyle{nonumberplain}
\theoremheaderfont{\bfseries}
\theorembodyfont{\normalfont}
\theoremsymbol{\mbox{$\Box$}}
\newtheorem{proof}{Proof}
}
\newtheorem{thm}{Theorem}
\newtheorem{lemma}{Lemma}

\usepackage[breaklinks=true,bookmarks=false]{hyperref}
\cvprfinalcopy 
\def\cvprPaperID{****} 

\begin{document}
\title{New insights on Multi-Solution Distribution of the P3P Problem}
\author{Bo Wang\\
Institute of Automation, Chinese Academy of Sciences \\
{\tt\small bo.wang@ia.ac.cn}
\and
Hao Hu \\
Ming Hsieh Department of Electrical Engineering, University of Southern California\\
{\tt\small huhao@usc.edu}
\and
Caixia Zhang \\
College of Science, North China University of Technology\\
{\tt\small zhangcx@ncut.edu.cn}
}
\maketitle
\begin{abstract}
Traditionally, the P3P problem is solved by firstly transforming its 3 quadratic equations into a quartic one, then by locating the roots of the resulting quartic equation and verifying whether a root does really correspond to a true solution of the P3P problem itself. However, a root of the quartic equation does not always correspond to a solution of the P3P problem .
 In this work, we show that when the optical center is outside of all the 6 toroids defined by the control point triangle, each positive root of the Grunert's quartic equation must correspond to a true solution of the P3P problem, and the corresponding P3P problem cannot have a unique solution, it must have either 2 positive solutions or 4 positive solutions. In addition, we show that when the optical center passes through any one of the 3 toroids among these 6 toroids ( except possibly for two concentric circles) , the number of the solutions of the corresponding P3P problem always changes by 1, either increased by 1 or decreased by 1.Furthermore we show that such changed solutions always locate in a small neighborhood of control points, hence the 3 toroids are critical surfaces of the P3P problem and the 3 control points are 3 singular points of solutions. A notable example is that when the optical center passes through the outer surface of the union of the 6 toroids from the outside to inside, the number of the solutions must always decrease by 1. Our results are the first to give an explicit and geometrically intuitive relationship between the P3P solutions and the roots of its quartic equation. It could act as some theoretical guidance for P3P practitioners to properly arrange their control points to avoid undesirable solutions.
\end{abstract}
\section{Introduction}
The Perspective-3-Point Problem, or P3P problem, is a single-view based pose estimation method. It was first introduced by Grunert \cite{Grunert1841} in 1841, and popularized in computer vision community a century later by mainly the Fishler and Bolles’ work in 1981\cite{RANSAC1981}. The P3P problem requires the least number of points to have a finite number of solutions and no feature-matching across views is needed. It has been widely used in various fields (\cite{high-precision2006},\cite{Kneip2011A},\cite{SHIQI2011A},\cite{Linnainmaa1988Pose},\cite{Lowe1991},\cite{Nister2007},\cite{Quan1999},\cite{Tang2009}). For its minimal demand in restricted working enrironment and compuational efficiency in RANSAC framework, the P3P problem is preferred due to its minimum requirement.
The P3P problem can be defined as: Given the perspective projections of three control points with known coordinates in the world system and a calibrated camera, to determine the position and orientation of the camera in the world system. It is shown that\cite{Haralick1994} the P3P problem could have 1,2,3 or at most 4 solutions depending on the configuration between the optical center and its 3 control points.
Since in any real applications, some basic questions must be answered before any real implementation, such as does it has a unique solution? If not, how many solutions could it have? Is the solution stable? etc., the multiple solution phenomenon in the P3P problem has been a focus of investigation since its very inception in the literature. Traditionally the multi-solution phenomenon in P3P problem is analyzed by at first transforming its 3 quadratic constraints into a quartic equation, then multiple roots of the quartic equation are located to derive possible solutions. For example, Haralick et al summarized 6 different transformation methods\cite{Haralick1994} . Gao \etal~ \cite{Gaoxiaoshan2003} gave a complete solution classification in an algebraic way. Gao \etal~ \cite{Gaoxiaoshan2006} also gave an analysis on the solutions distribution in a probabilistic way. Recently, Rieck (\cite{Rieck2012JMIV},\cite{Rieck2012VISAPP},\cite{Rieck2014}) gave a systematic analysis on the multi-solution phenomenon via his introduced novel algebraic entities, in particular on the distribution of multiple solutions around the danger cylinder and their stability problem. From a geometric way, Zhang and Hu\cite{cylinder2006} showed that when the optical center lies on the danger cylinder, the P3P problem must have 4 solutions, one is a double solution, and when the optical center lies on the three vertical planes, the P3P problem must have a pair of side-sharing solutions and a pair of point-sharing solutions \cite{four2005}; Sun and Wang \cite{sunfengmei2010} gave an interpretation of the solution changes at the intersecting lines of the three vertical planes with the danger cylinder. Wu and Hu \cite{wuyihong2006} gave a thorough investigation on the degenerate cases.
Although many results are reported in these literature, some key questions still persist: for example, it is well known that a root of the resulting quartic equation does not necessarily correspond to a solution of the P3P problem itself. It could in fact correspond to 2, 1 solutions, even no solution at all. In other words, a one-to-one relationship between a root and a solution does not exist in general case, but under what conditions does some definitive relation exist? It is still an open question up to now.Another essential question is: what is the relationship between the number of solutions and the position of optical center in geometric view? In this work, we try to give answers to the above questions, our main contributions are 3-fold:\\
\indent (1). We show that given 3 control points, when the optical center is outside of the 6 toroids, the corresponding P3P problem cannot have a unique solution, it can only have either 2 solutions or 4 solutions.\\
\indent (2). We show that given 3 control points, when the optical center is outside of the 6 toroids, each positive root of the Grunert’s quartic equation must correspond to a positive solution of the P3P problem, or in this restricted case, a one-to-one relationship does exist between a root and a solution, and our result seems the first in the literature to establish such a one-to-one relationship. In addition, this conclusion holds also for the other 5 different methods by Haralick\cite{Haralick1994}.\\
\indent (3). We show that given 3 control points , the 6 toroids can be divided into two groups, when the optical center passes through a toroid in one group, the number of the solutions of the corresponding P3P problem must change exactly by one, either increased by one, or decreased by one. However when the optical center passes through a toroid in the other group, the number of the solutions of the corresponding problem does not change, but its S-solution( see the main text for its definition) changes from one to another. This result shows that the 3 toroids must act as some kinds of critical surfaces in term of the solution distribution of the P3P problem.\\
\indent The above three results provide some new insights into the nature of the multi-solution phenomenon in P3P problem, and could be also of theoretical guidance to P3P practitioners in addition to their academic values.
The paper is organized as follows: The next section is some preliminaries as well as some new concepts, including the supplementary P3P problem and supplementary-solution, or S-solution; The main results are reported in Section 3, and Section 4 is some concluding remarks.
\hfill

\section{Preliminaries}
For the notational convenience, some preliminaries on the P3P problem are provided at first.
\subsection{P3P Problem and Its Corresponding Supplementary Problems}
\begin{figure}[t]
\begin{center}
\includegraphics[width=0.3\linewidth]{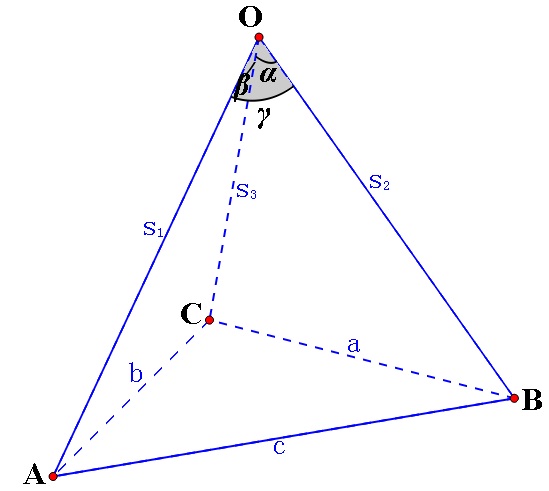}
\end{center}
   \caption{P3P problem definition: $A,B,C$ are the 3 control points, $O$ is the optical center.}
\label{Fig1}
\end{figure}
As shown in Fig.1, $A,B,C$ are the three control points with known distance $a=|BC|$, $b=|AC|$, $c=|AB|$, $O$ is the optical center. Since the camera ( under the pinhole model) is calibrated, the three subtended angles $\alpha,\beta,\gamma$ of the projection rays can be considered known entities, then by the Law of Cosines, the following 3 basic constraints on the three unknown  $s_1=|OA|,s_2=|OB|,s_3=|OC|$ in (1) must hold.
\begin{gather}\label{equ1}
  \begin{cases}
    s_1^2 +s_2^2 -2cos(\gamma)s_1s_2 =c^2 \\
    s_1^2 +s_3^2 -2cos(\beta)s_1s_3 =b^2  \\
    s_2^2 +s_3^2 -2cos(\alpha)s_2s_3 =a^2
  \end{cases}
\end{gather}
\indent Hence the P3P problem is meant to determine such positive triplets $(s_1,s_2,s_3)$ satisfying the 3 basic constraints in (\ref{equ1}).\\
\indent Given a P3P problem in (\ref{equ1}), we can define its $(\alpha,\beta)$ \textbf{supplementary problem} in (\ref{equ2}):
\begin{gather}\label{equ2}
  \begin{cases}
    s_1^2 +s_2^2 -2cos(\gamma)s_1s_2& =c^2 \\
    s_1^2 +s_3^2 -2cos(\pi-\beta)s_1s_3 &=b^2  \\
    s_2^2 +s_3^2 -2cos(\pi-\alpha)s_2s_3 &=a^2
  \end{cases}
\end{gather}
\indent Since the constraint system (\ref{equ2}) is obtained by replacing  $\alpha$ and $\beta$ in (\ref{equ1}) by their respective supplementary angle ($\pi-\alpha$) and ($\pi-\beta$), the corresponding problem is called the $(\alpha,\beta)$ supplementary problem of the original P3P problem in this work. Similarly, the $(\beta,\gamma)$ supplementary problem, and  $(\alpha,\gamma)$ supplementary problem can be defined.
\subsection{Solution and S-Solution of the P3P Problem}
A P3P solution is a positive triplet ($s_1>0,s_2>0,s_3>0 $) satisfying the 3 constraints in (\ref{equ1}) since any distance must be positive. Note that if a real-valued triplet($s_1,s_2,s_3 $) satisfying (\ref{equ1}) and if $s_1\neq 0,s_2\neq0,s_3\neq0 $, the signs of $s_1,s_2,s_3 $ could be:\\
(1): all three positive; (2): all three negative; (3): two positive and one negative; (4): two negative and one positive.

\indent In addition, since if a triplet ($s_1,s_2,s_3 $) satisfies (\ref{equ1}), ($-s_1,-s_2,-s_3 $) must also satisfy (\ref{equ1}), we need only consider those triplets with all three positive elements or two positive elements + one negative element.
  As we have said, a positive triplet satisfying (\ref{equ1}) is a solution of the P3P problem. For a triplet ($s_1,s_2,s_3 $) with one negative element + two positive elements and satisfying (\ref{equ1}), say ($s_1>0,s_2>0,s_3<0 $), its positive counterpart ($s_1,s_2,-s_3 $) must be a solution of its $(\alpha,\beta)$ supplementary problem (\ref{equ2}), hereinafter, such a triplet ($s_1,s_2,s_3 $) is called a \textbf{supplementary solution}, in short, a \textbf{S-solution}, of the original P3P problem.
  In sum, a non-zero-real-valued triplet satisfying (\ref{equ1}) corresponds to either a solution or a S-solution of the original P3P problem.
\subsection{The 6 Toroids}
\begin{figure}[t]
\begin{center}
\includegraphics[width=0.3\linewidth]{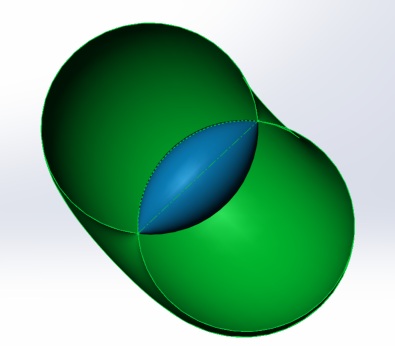}
\includegraphics[width=0.3\linewidth]{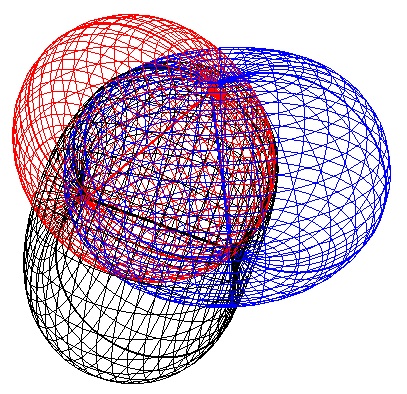}
\end{center}
   \caption{By rotating the circumcircle of the triangle ABC around one of its three sides separately, 3 pairs of toroids can be generated;Left: two toroids($T_{\theta},T_{\pi-\theta}$); Right: Three toroids.}
\label{Fig2}
\end{figure}
As shown in Fig.2, by rotating the circumcircle of the triangle $ABC$ around one of its three sides separately, 3 pairs of toroids can be generated \cite{Wolfe1991}, denoted as ($T_{\angle A},T_{\pi-\angle A}$),($T_{\angle B},T_{\pi-\angle B}$),($T_{\angle C},T_{\pi-\angle C}$)(see their definitions in the next section). In section 3 we will show that these 6 toroids play a very important role for the multiple solution phenomenon in the P3P problem.
\subsection{Grunert's Quartic Equation and Its Possible Root Distributions}
As shown in \cite{Haralick1994}, by setting $s_2=us_1,s_3=vs_1$ and by some algebraic manipulations, a quartic equation in $v$ can be derived from (\ref{equ1}) as:
\begin{equation}\label{Grunert’s Quartic Equation}
  A_4v^4+A_3v^3+A_2v^2+A_1v+A_0=0
\end{equation}
where the first and last coefficients have the following form:
\begin{gather}
  A_4=4c^2/b^2(cos^2\angle A-cos^2\alpha)  \\ 
  A_0=4a^2/b^2(cos^2\angle C-cos^2\gamma)  \label{A0}
\end{gather}
the above (4) indicates that if the optical center $O$ is not on the toroid pair ($T_{\angle A},T_{\pi-\angle A}$), $A_4\neq 0$. If $O$ is outside of their union, then $\alpha <min(\angle A,\pi-\angle A),A_4<0$. Otherwise if $O$ is inside of their union, $ min(\angle A,\pi-\angle A)<\alpha<max(\angle A,\pi-\angle A),A_4>0$.
Similarly, (\ref{A0}) indicates that if the optical center $O$ is not on toroid pair ($T_{\angle C},T_{\pi-\angle C}$), $A_0\neq 0$. If $O$ is outside of their union, then $\gamma <min(\angle C,\pi-\angle C),A_0<0$. Otherwise if $O$ is inside of their union, $ min(\angle C,\pi-\angle C)<\gamma<max(\angle C,\pi-\angle C),A_0>0$.
Clearly if the optical center lies on either outside of the union of the 4 toroids  ($T_{\angle A},T_{\pi-\angle A}$),($T_{\angle C},T_{\pi-\angle C}$), or inside of their intersection, $A_4\times A_0>0$ always hold, which in turn implies that the quartic equation in (\ref{Grunert’s Quartic Equation}) can have either two real roots + a pair of complex roots, or 4 real roots, including possible repeated roots.

\subsection{Notational Definitions}
 For the notational convenience, we first introduce some definitions.
\boldmath
\begin{itemize}
  \item ${T_{\angle A}(T_{\pi-\angle A},T_{\alpha},T_{\pi-\alpha})}$: The toroid generated by rotating the circular arc with the fixed inscribed angle of $\angle A(\pi-\angle A, \alpha,\pi-\alpha)$ around the circular segment BC by $2\pi$; Similarly,there are ${T_{\angle B}(T_{\pi-\angle B},T_{\beta},T_{\pi- \beta})}$, ${T_{\angle C}(T_{\pi-\angle C},T_{\gamma},T_{\pi- \gamma})}$.
  \item ${Ball_{ABC}}$: The circumsphere of the P3P problem, defined as the sphere such that the circumcircle of the P3P problem is a great circle of the sphere.
  \item $\overline{S}$: The set of the points inside and on the closed surface $S$.
  \item ${\overline{T_{union}}}$: The union of the 6 point sets:$T_{\angle A}$, $T_{\pi-\angle A}$, $T_{\angle B}$, $T_{\pi-\angle B}$, $T_{\angle C}$, $T_{\pi-\angle C}$.
  \item ${n_{T_{\angle A}}}$: The outer normal vector of the toroid $T_{\angle A}$ at point $A$ ;
  \item ${L_{\alpha,\beta}}$: The intersecting curve of the two toroids: $T_\alpha$ and $T_\beta$.
\end{itemize}

\section{Main Results}
 Some new results are introduced which include five theorems. They are some geometric conditions on distribution of a unique solution and changes of the numbers of solutions.

%
\subsection{The Role of $\overline{T_{union}}$ to the Number of the P3P Multi-Solution}
In the following theorem 1 and 2, we show that when the optical center locates outside of Tunion, the P3P problem cannot have a unique solutiosn, and each positive root of the Grunert's derivation correspond to a solution of the P3P problem.
\begin{thm}\label{no unique}
  If the optical center $O$ is outside of $\overline{T_{union}}$, the corresponding P3P problem cannot have a unique solution, it must have at least two solutions.
\end{thm}
Before giving a formal proof of Theorem 1, we introduce at first the following three lemmas:
\begin{lemma}\label{Ball}
  Assume $P$ is an arbitrary point on the circumsphere, $Ball_{ABC}$, of the P3P problem ,then the subtended angles of P to the three control points ($A,B,C$), $\angle APC$  must satisfy the following condition:
  $ min(\angle B,\pi-\angle B)<\angle APC <max(\angle B,\pi-\angle B)$.
\end{lemma}
\begin{proof}
  Since the circumcircle of the triangle $ABC$ is a great circle of $Ball_{ABC}$, its radius must be no less than that of the circumcircle of the triangle $APC$. As shown in Fig.\ref{no solution}, for any point on the circumcircle of the triangle $APC$, the above inequalities must hold.
\end{proof}
 \begin{lemma}\label{no S-solution}
   If the optical center $O$ is outside of $\overline{T_{union}}$, the corresponding P3P problem has no S-solutions.
 \end{lemma}
\begin{figure}[t]
\begin{center}
\includegraphics[width=0.3\linewidth]{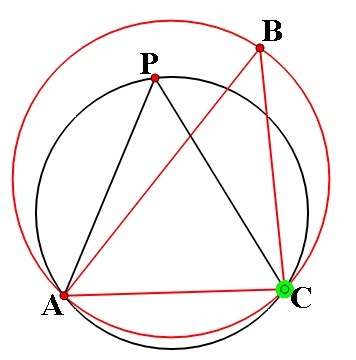}
\includegraphics[width=0.3\linewidth]{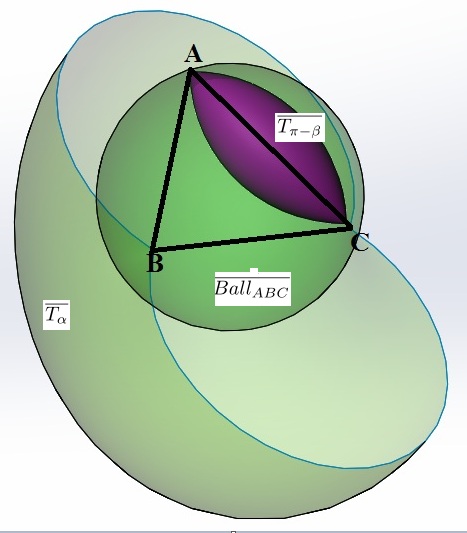}
\end{center}
   \caption{Left: $\angle$ $APC$ and $\angle$ $B$,$\pi-$ $\angle$ $B$, Right:$\overline{T_{\pi-\beta}}\subset \overline{Ball_{ABC}} \subset \overline{T_{\alpha}}$.}
\label{no solution}
\end{figure}
\begin{proof}
  Proofs by Contradiction: Suppose the optical center $O_1$ is outside of $\overline{T_{union}}$, and the corresponding P3P problem did have a S-solution.
  Without loss of generality, assume the S-solution is ($s_1<0,s_2>0,s_3>0 $), by definition of S-solution, the triplet ($s_1,s_2,s_3 $) must simultaneously satisfy the following two equivalent constraint systems in (\ref{equ3}) and (\ref{equ4}):
  \begin{gather}\label{equ3}
  \begin{cases}
    s_1^2 +s_2^2 -2cos(\gamma)s_1s_2 =c^2 \\
    s_1^2 +s_3^2 -2cos(\beta)s_1s_3 =b^2  \\
    s_2^2 +s_3^2 -2cos(\alpha)s_2s_3 =a^2
  \end{cases}
\end{gather}
\begin{gather}\label{equ4}
  \begin{cases}
    s_1^2 +s_2^2 -2cos(\pi-\gamma)(-s_1)s_2& =c^2 \\
    s_1^2 +s_3^2 -2cos(\pi-\beta)(-s_1)s_3 &=b^2  \\
    s_2^2 +s_3^2 -2cos(\alpha)s_2s_3 &=a^2
  \end{cases}
\end{gather}
Eq.(\ref{equ4}) indicates that the three toroids, $T_{\pi-\gamma}$, $T_{\pi-\beta}$ and $T_{\alpha}$ must at least intersect at one point which corresponds to ($-s_1,s_2,s_3$) as a solution of the supplementary problem . In the following, we show that such a common intersecting point does not exist, hence a S-solution does not exist.
Suppose $P$ is an arbitrary point on the circumsphere $Ball_{ABC}$ of the P3P problem, according to Lemma \ref{Ball}, $ min(\angle B,\pi-\angle B)<\angle APC <max(\angle B,\pi-\angle B)$.Since $O_1$ is outside of $\overline{T_{union}}$, we have:
\begin{gather*}
   \pi-\beta>max(\angle B,\pi-\angle B),
   \pi-\gamma>max(\angle C,\pi-\angle C)
\end{gather*}
Since $\pi-\beta>max(\angle B,\pi-\angle B)$,we have $\overline{T_{\pi-\beta}}\subset \overline{Ball_{ABC}}$, as shown in Fig.\ref{no solution}. Similarly since $ \alpha<min(\angle A,\pi-\angle A)$, we have $\overline{Ball_{ABC}}\subset \overline{T_{\alpha}}$. In addition, since $\overline{T_{\pi-\beta}}$ and $ \overline{Ball_{ABC}}$ have only two common points $A$ and $C$, and $\overline{Ball_{ABC}}$ and $ \overline{T_{\alpha}}$ have only two common points $B$ and $C$, by $\overline{T_{\pi-\beta}}\subset \overline{Ball_{ABC}} \subset \overline{T_{\alpha}}$, $\overline{T_{\pi-\beta}} $ and $ \overline{T_{\alpha}}$ have only a single common intersecting point $C$.
Similarly we can prove that $\overline{T_{\pi-\gamma}} $ and $ \overline{T_{\alpha}}$ can only have a single intersecting point at B.Hence the three toroids ($ \overline{T_{\alpha}}$,$\overline{T_{\pi-\beta}} $ ,$\overline{T_{\pi-\gamma}} $) has no real intersecting point. In other words, system (\ref{equ4}) has no any positive solution, or a S-solution of the P3P problem in (\ref{equ3}).
\end{proof}
\begin{lemma}\label{Grunert real}
\item[(1)] Given a \emph{real} root of the Grunert's quartic equation in Eq.(\ref{Grunert’s Quartic Equation}) in Section 2.4, there always exists a corresponding \emph{real} triplet($s_1,s_2,s_3$) satisfying (\ref{equ1});
\item[(2)] If $s_1=0$ (or $s_2=0$, or $s_3=0$) in the above triplet ($s_1,s_2,s_3$), then the corresponding optical center must lie on one of the toroid pair ($T_{\angle A}$, $T_{\pi-\angle A}$)(or ($T_{\angle B}$, $T_{\pi-\angle B}$), or ($T_{\angle C}$, $T_{\pi-\angle C}$));
\item[(3)] If $s_1\neq 0$,$s_2\neq 0$ and $s_3 \neq 0$, then the triplet ($s_1,s_2,s_3$)  must be either a solution or a S-solution of the P3P problem.
\end{lemma}
\begin{proof}
 ~\item[(1)]From the derivation process of the Grunert's quartic equation (\ref{Grunert’s Quartic Equation}) in \cite{Haralick1994}, the first part can be derived directly. The details are omitted;
 ~\item[(2)] If $s_1=0$, then by Eq.(\ref{equ1}),$s_2=\pm c,s_3=\pm b$ and $c^2 +b^2 -2cos(\alpha)(\pm c) (\pm b) =a^2$, hence $cos(\alpha)=\pm cos\angle A$. That is, the optical center should lie on a curve on one of the toroids ($T_{\angle A}$, $T_{\pi-\angle A}$). Similarly when $s_2=0$ (or $s_3=0$), the optical center should lie on ($T_{\angle B}$, $T_{\pi-\angle B}$)( or ($T_{\angle C}$, $T_{\pi-\angle C}$));
 ~\item[(3)]As shown in Section 2.4, if the triplet ($s_1,s_2,s_3$) satisfies the constraint system (\ref{equ1}), and $s_1\neq 0$,$s_2\neq 0$ and $s_3 \neq 0$, then either all the three elements in ($s_1,s_2,s_3$) are positive, in this case it is a solution of the P3P problem, or one element is negative and the other two are positive, in this case, it is a S-solution of the P3P problem.
\end{proof}
\textbf{Proof of Theorem 1}
\begin{proof}
As shown in Section 2.4, given a P3P problem, its Grunert’s quartic equation can have either two distinct real roots + a pair of complex roots or 4 real roots, hence when the optical center $O$ is outside of $\overline{T_{union}}$, by (2) of Lemma \ref{Grunert real}, $s_1 \neq 0,s_2 \neq 0,s_3 \neq 0$ , furthermore, by (1) of Lemma \ref{Grunert real}, the sum of the number of solutions and S-solutions (including repeated solution) of the P3P problem can only be 2 or 4. Then By Lemma \ref{no S-solution}, S-solution is impossible, hence the number of the solutions of the P3P problem must be either 2 or 4. That is, a unique solution is impossible.
\end{proof}
\textbf{Remark 1}
A direct consequence of Theorem \ref{no unique} is that from a statistic point of view, the probability of a P3P problem having a unique solution is zero. This is because by removing the scale factor, the 3 control points of any P3P problem can be considered lying on a unit circle. Given 3 control points on the unit circle, by Theorem \ref{no unique}, the unique solution can only occur when the optical center is inside of or on $\overline{T_{union}}$. Since the 3D space within $T_{union}$ is a fixed one, and the space outside of $\overline{T_{union}}$ is infinitely large, the probability of the P3P problem having a unique solution must be zero.
\begin{thm}\label{one to one}
  If the optical center O is outside of $\overline{T_{union}}$, under the Grunert’s derivation, each positive root of the quartic equation must correspond to a solution of the P3P problem.
\end{thm}
\begin{proof}
  When the optical center O is outside of $\overline{T_{union}}$, by (2) of Lemma \ref{Grunert real}, none of the 3 elements is zero in the triplet ($s_1,s_2,s_3$). By (3) of Lemma \ref{Grunert real}, this non-zero triplet ($s_1,s_2,s_3$) should be either a solution of the P3P problem, or a S-solution of the P3P problem. By Lemma \ref{no S-solution}, when the optical center $O$ is outside of $\overline{T_{union}}$ , it cannot be a S-solution, hence it must be a solution of the P3P problem.
\end{proof}
\textbf{Remark 2}
\begin{itemize}
\item[(1)] Theorem \ref{one to one} provides an explicit relation between a solution of the P3P problem and a positive root of the resulting quartic equation. To our knowledge, our result is the first in the literature to establish such a relation;
\item[(2)] Theorem \ref{one to one} indicates that when the optical center is outside of $\overline{T_{union}}$,  a simple relationship exists between a positive root of the quartic equation and a solution of the P3P problem.
\end{itemize}

\subsection{When the optical center passes through one of the 3 toroids ($T_{\angle A},T_{\angle B},T_{\angle C}$)}

In the following, we investigate the possible changes of the number of the solutions in the P3P problem when the optical center passes through one of the 3 toroids ($T_{\angle A},T_{\angle B},T_{\angle C}$). We have the following result:
\begin{thm}\label{change}
  When the optical center $O$ passes through toroid $T_{\angle A}(T_{\angle B},T_{\angle C})$ except possibly for two concentric circles on it, the number of the solutions of the corresponding P3P problem must change exactly by 1, either increased by 1 or decreased by 1.
\end{thm}
Here, we only give a proof for the case of the optical center passing through the toroid $T_{\angle A}$, the other two can be similarly proved. At first, we introduce the following 5 lemmas:
\begin{lemma}\label{Lemma4}
  In a small neighborhood of point $A$ (or $B$), the toroid $T_\gamma$ could approach infinitesimally to the right circular cone $Cone_{A,AB,(\pi-\gamma)}$ (or  $Cone_{B,AB,(\pi-\gamma)}$) with the point $A$ ( or $B$) as the apex point, line $AB$ the axis, and $(\pi-\gamma)$ the half-angle under any $L_i  (i=1,2,\ldots,\infty)$ norm. In other words, $ \forall \epsilon_1>0, \exists \epsilon_A>0$, if $X\in \delta_A=\{X|\|X-A\| \leq \epsilon_A \}$, we have $\|Cone_{A,AB,(\pi-\gamma)}-T_\gamma\| \leq\epsilon_1.$
\end{lemma}
\begin{figure}[t]
\begin{center}
\includegraphics[width=0.6\linewidth]{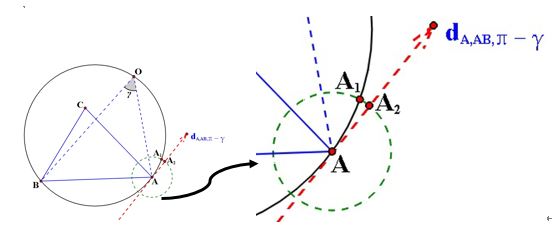}
\includegraphics[width=0.3\linewidth]{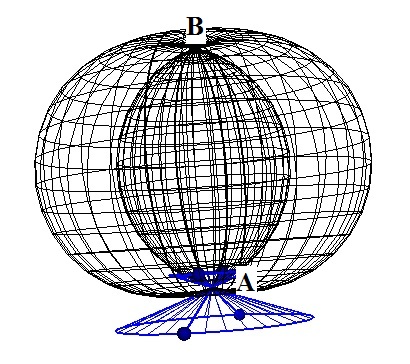}
\end{center}
   \caption{Toriod can be approximated by a cone at the apex.}
\label{fig4}
\end{figure}
\begin{proof}
  As shown in Fig.\ref{fig4}, $T_\gamma$ is generated by rotating the arc $\widehat{AOB}$ around cord $AB$, $Cone_{A,AB,(\pi-\gamma)}$ is generated by rotating the tangent line $d_{A,AB,(\pi-\gamma)}$ of the circle $AOB$ at point $A$ around cord $AB$. Without loss of generality, let us only consider the plane $AOB$ in Fig.\ref{fig4}a: Given a small circular neighborhood at point $A$, say $\delta_A=\{X|\|X-A\| \leq \epsilon_A \}$ , the outer border of $\delta_A$ intersects the arc $\widehat{AOB}$ at point $A_1$, and intersects the tangent line $d_{A,AB,(\pi-\gamma)}$ at point $A_2$, then we have:
  \begin{gather}
  \begin{split}\label{limit to cone}
  \lim\limits_{\epsilon_A \to 0 }\frac {\|A_1-A_2\|_i}{\epsilon_A }
  &=\lim\limits_{\epsilon_A \to 0 }\frac {\|2R_\gamma sin(arctan(\frac{\epsilon_A}{4R_\gamma}))\epsilon_A\|_i}{\epsilon_A}\\
 &=\lim\limits_{\epsilon_A \to 0 }\frac{\| \frac{\epsilon_A^2}{2}\|_i}{\epsilon_A}=0
  \end{split}
  \end{gather}
where $R_\gamma$ is the radius of the circle $AOB$. (\ref{limit to cone}) is meant that the toroid $T_\gamma$ approach infinitesimally to the circular cone $Cone_{A,AB,(\pi-\gamma)}$.
A geometric illustration for the situation is shown in Fig.\ref{fig4}b.
\end{proof}
\begin{figure}[t]
\begin{center}
\includegraphics[width=0.6\linewidth]{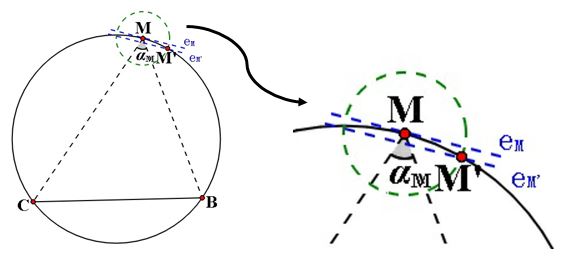}
\end{center}
   \caption{$d_{EM}=O(d_M)$.}
\label{fig6}
\end{figure}
\begin{lemma}  \label{Lemma5}
  As shown in Fig.\ref{fig6}, $M$ is a point on the circular arc $\widehat{BMC}$ generated by the fixed inscribed angle of $\alpha$ with respect to the cord $BC$, the radius of $\widehat{BMC}$ is $R_\alpha $, $e_M$ the tangent line at $M$, then given a small circular neighborhood at $M$, $\delta_M=\{X|\|X-M\|\leq \epsilon_M \}$, the outer border of $\delta_M$ intersects with  $\widehat{BMC}$ at point $M'$. Define the distance from $M'$ to $e_M$ as $d_{EM}$, and the distance between $M'$ and $M$ as $d_M$, then we have
  $$\lim \limits_{\epsilon_M \to 0 }\frac {d_{EM}}{d_M}=0$$
\end{lemma}
\begin{proof}
  When $\epsilon_M \rightarrow 0 $, we have
{\small
  \begin{equation*} \begin{split}
    d_{EM}
  &\rightarrow R_\alpha-R_\alpha \cos \left(\frac{d_M}{R_\alpha}\right ) \\
  &\rightarrow R_\alpha-R_\alpha \left(1-\frac{1}{2}\left(\frac{d_M}{R_\alpha}\right)^2\right)=\frac {d_M^2}{2R_\alpha}
  \end{split}
  \end{equation*}
  }
Hence,
{\small
$$\lim \limits_{\epsilon_M \to 0 }\frac {d_{EM}}{d_M}
=\lim \limits_{\epsilon_M \to 0 }\frac {d_{M}}{2R_\alpha}
=\lim \limits_{\epsilon_M \to 0 }\frac {\epsilon _M}{2R_\alpha}=0$$
}\end{proof}
\begin{figure}[t]
\begin{center}
\includegraphics[width=0.6\linewidth]{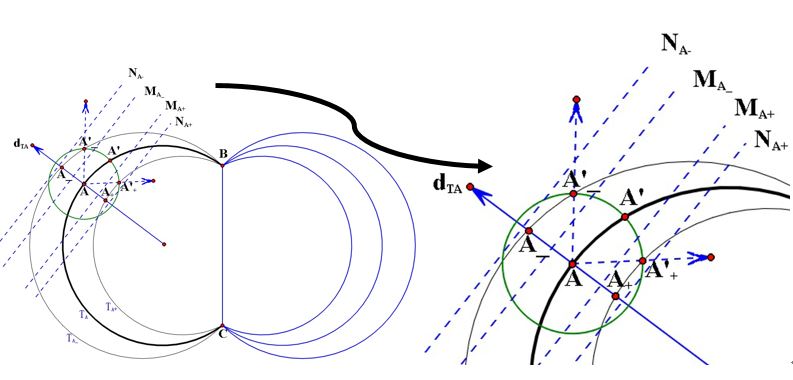}
\end{center}
   \caption{$T_A$,$T_{A-}$ and $T_{A+}$ almost parallel at point $A$ within a small neighborhood.}
\label{fig7}
\end{figure}
Now let $A_+$ denote a point which is inside of the toroid $T_{\angle A}$ with the larger inscribed angle $\angle A_+$ than $\angle A$,
 and $A_-$ a point which is outside of the toroid $T_{\angle A}$ with the smaller inscribed angle $\angle A_-$ than $\angle A$.
\begin{lemma} \label{Lemma6}
  As shown in Fig.\ref{fig7}, $\forall \epsilon_\theta>0, \exists \epsilon_A>0$ and $\epsilon_{\angle A} >0$ such that $|\angle A_+-\angle A|=\Delta \angle A_+< \angle \epsilon_A $ and  $|\angle A_--\angle A|=\Delta \angle A_-< \angle \epsilon_A $ for the toroids $T_{\angle A+}$ and $T_{\angle A-}$ within the neighborhood of point $A:\delta_A= \{X|\|X-A\| \leq \epsilon_A \}$. In addition, denote $\theta_{T_{\angle A-} }$($\theta_{T_{\angle A+} }$) as the included angle between the outer normal at point $A, d_{T_{\angle A}}$, with the segment $AA_-'(AA_+')$, where $A_-'(A_+')$ is an intersecting point between the toroid $T_{\angle A-}$($T_{\angle A+}$)
with the outer border of $\delta_A$, then for any given  $(\theta_{g-}\in(0,\frac{\pi}{2} ),\theta_{g+}\in (\frac{\pi}{2},\pi))$, if  $0<\frac{\pi}{2}-\theta_{g-}<\epsilon_\theta$ ($0<\theta_{g+}-\frac{\pi}{2}<\epsilon_\theta$), we have $\theta_{g-}<\theta_{T_{\angle A-}}<\frac{\pi}{2}$($\frac{\pi}{2}<\theta_{T_{\angle A+}}<\theta_{g+}$).
\end{lemma}
\begin{proof}
  By Lemma \ref{Lemma5}, we know that if $\epsilon_A$ is sufficiently small, the intersecting point $M'$ of the outer border $\delta_A= \{X|\|X-A\| \leq \epsilon_A \}$ with the toroid $T_{\angle A}$ must be within the two parallel planes $\Pi_{MA_-}$ and $\Pi_{MA_+}$ , where both $\Pi_{MA_-}$ and $\Pi_{MA_+}$ are orthogonal to the outer normal $d_{T_{\angle A}}$ at point $A$, and their distance is a higher order infinitesimal of $\epsilon_A$. Similarly if $\epsilon_{\angle A}$ is sufficiently small, $A_-'$ ($A_+'$) should be within the two parallel planes $\Pi_{MA_-}$ and $\Pi_{NA_-}$($\Pi_{MA_+}$ and $\Pi_{NA_+}$) with $\Pi_{MA_-}$ and $\Pi_{NA_-}$ ($\Pi_{MA_+}$ and $\Pi_{NA_+}$) both being orthogonal to $d_{T_{\angle A}}$ and their distance also being a higher order infinitesimal of $\epsilon_A$ . Since all the 3 distances between $\Pi_{MA_-}$ and $\Pi_{MA_+}$, $\Pi_{MA_-}$ and $\Pi_{NA_-}$ and $\Pi_{MA_+}$ and $\Pi_{NA_+}$ are of higher order infinitesimal of $\epsilon_A$, then for any given   $(\theta_{g-}\in(0,\frac{\pi}{2} ),\theta_{g+}\in (\frac{\pi}{2},\pi))$, if  $0<\frac{\pi}{2}-\theta_{g-}<\epsilon_\theta$ ($0<\theta_{g+}-\frac{\pi}{2}<\epsilon_\theta$), we have $\theta_{g-}<\theta_{T_{\angle A-}}<\frac{\pi}{2}$($\frac{\pi}{2}<\theta_{T_{\angle A+}}<\theta_{g+}$).
\end{proof}
\begin{figure}[t]
\begin{center}
\includegraphics[width=0.3\linewidth]{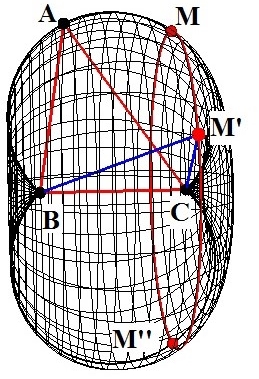}
\end{center}
   \caption{The optical center $M'$ lies on the $T_{A}$.}
\label{fig8}
\end{figure}
\begin{lemma}\label{Lemma7}
  As shown in Fig.\ref{fig8}, $M'$ is a point on the toroid $T_{\angle A}$, $\angle AM'C=\beta$ is the subtended angle of $M'$ to the chord $AC$, and  $\angle AM'B=\gamma$ the subtended angle of $M'$ to the chord $AB$ . By lemma \ref{Lemma4} at a small neighborhood of point $A$, the toroid $T_\beta$ and $T_\gamma $ can be infinitesimally approximated as the $Cone_{A,AC,(\pi-\beta)}$ and $Cone_{A,AB,(\pi-\gamma)}$, then these two cones must intersect at two rays from point $A$.
\end{lemma}
\begin{proof}
  Since $\gamma_0=(\pi-\gamma)$ is the half angle of $Cone_{A,AB,\gamma_0}$, and $\beta_0=(\pi-\beta)$ the half angle of $Cone_{A,AC,\beta_0}$, and since $\gamma_0,\beta_0 \in [0,\pi]$  by definition, then depending on the relationship among the three angles $\gamma_0,\beta_0$ and $\angle A$, $\pi-\angle A$, the necessary and sufficient condition for the two cones to intersect is:
{\tiny
\begin{equation*}
    \begin{cases}
  \beta_0+\gamma_0 \geq \angle A & \text{if}
\begin{cases}
  \beta_0 \leq \angle A\\
  \gamma_0 \leq \angle A
\end{cases}
\\
 \beta_0-\gamma_0 \leq \angle A & \text{if}
\begin{cases}
  \beta_0 \geq \angle A\\
  \gamma_0 \leq \pi-\angle A
\end{cases}
\\
 \gamma_0-\beta_0 \leq \angle A & \text{if}
\begin{cases}
  \beta_0 \leq \pi-\angle A\\
  \gamma_0 \geq \angle A
\end{cases}
\\
  \beta_0+\gamma_0 \leq 2\pi-\angle A & \text{if}
\begin{cases}
  \beta_0 \geq \pi-\angle A\\
  \gamma_0 \geq \pi-\angle A
\end{cases}
\end{cases}
\end{equation*}
}
The geometrical meaning is shown as Fig.\ref{fig9}.
\begin{figure}[t]
\begin{center}
\includegraphics[width=0.2\linewidth]{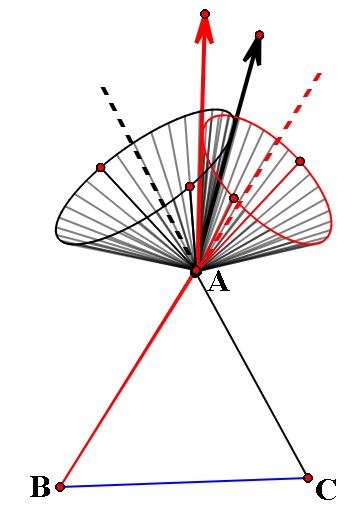}
\includegraphics[width=0.2\linewidth]{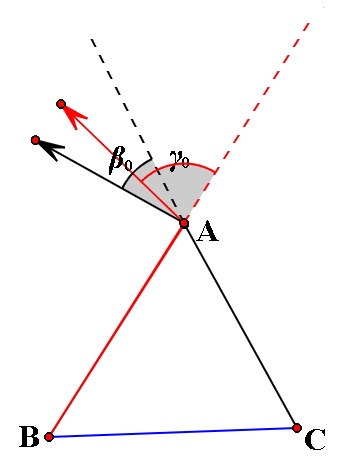}
\includegraphics[width=0.2\linewidth]{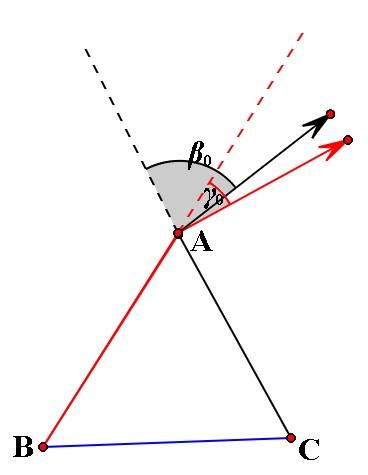}
\includegraphics[width=0.2\linewidth]{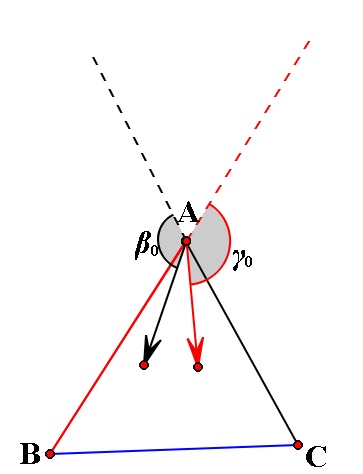}
\end{center}
   \caption{two cones must intersect at the point $A$.}
\label{fig9}
\end{figure}
The above 4 cases can be equivalently expressed as:
\begin{gather}
    \begin{cases}\label{inequa}
      \angle A\leq \beta_0+\gamma_0 \leq 2\pi-\angle A \\
      |\beta_0-\gamma_0| \leq \angle A
     \end{cases}
\end{gather}%
\begin{figure}[t]
\begin{center}
\includegraphics[width=0.3\linewidth]{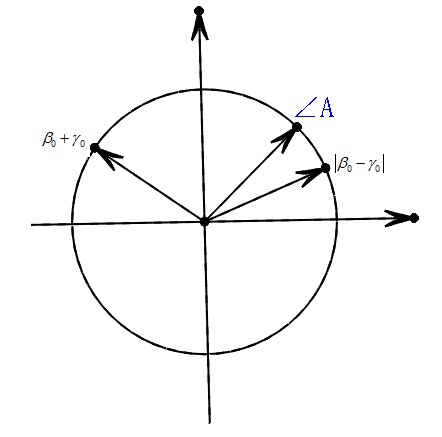}
\includegraphics[width=0.3\linewidth]{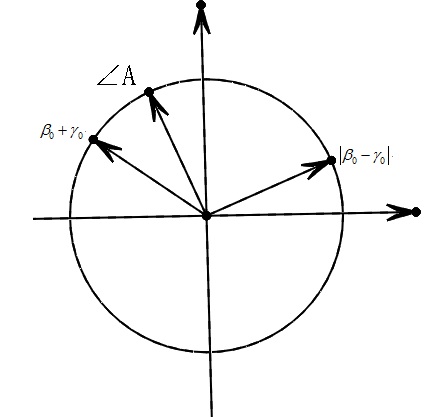}
\end{center}
   \caption{The relationship between angle $A$ and $\beta,\gamma$.}
\label{fig10}
\end{figure}
The constraints in (\ref{inequa}) can be geometrically interpreted in Fig.\ref{fig10}. The left is for $\angle CAB \leq \frac {\pi}{2}$, and The right is for $\angle CAB > \frac {\pi}{2}$. Then for both the two cases in Fig.\ref{fig10}, the constraints in (\ref{inequa}) can be simplified as:
\begin{gather}
    \label{intersect}
     cos(\beta_0+\gamma_0) \leq cos\angle A \leq  cos(\beta_0-\gamma_0)
\end{gather}
(\ref{intersect}) can be further expressed as:
\begin{gather}
\begin{split}\label{equivelent}
  cos^2\beta_0+cos^2\gamma_0-sin^2\angle A-2cos\beta_0cos\gamma_0cos\angle A\leq 0
\end{split}
\end{gather}
Since $M'$ is obtained by rotating a point $M$ on the circular arc $\widehat{BAC}$ around $BC$, suppose $M''$ is the rotated point of $M$ around $BC$ by $\pi$, then we have
 \begin{equation*} \begin{split}\label{equivelent1}
  &cos\beta= \frac{\|M'A\|^2+\|M'C\|^2-\|CA\|^2}{2\|M'A\|\|M'C\|}\\
  &cos\gamma=\frac{\|M'A\|^2+\|M'B\|^2-\|BA\|^2}{2\|M'A\|\|M'B\|}
  \end{split}
  \end{equation*}
Since $\gamma_0=\pi-\gamma$,$\beta_0=\pi-\beta$, by substituting $cos(\pi-\beta)$ and $cos(\pi-\gamma)$ into (\ref{equivelent}) and by some manipulations, we have:
{\small
\begin{gather}
\begin{split}\label{equivelent2}
 &\|M'B\|^2(\|M'A\|^2+\|M'C\|^2-\|AC\|^2 )^2 \\
&+\|M'C\|^2 (\|M'A\|^2+\|M'B\|^2 -\|AB\|^2 )^2\\
&-4 sin^2\angle A \|M'A\|^2 \|M'B\|^2 \|M'C\|^2\\
&-2\|M'C\|\|M'B\|(\|M'A\|^2+\|M'C\|^2-\|AC\|^2 )\cdot\\
&(\|M'A\|^2+\|M'B\|^2 -\|AB\|^2 )cos\angle A\leq 0
\end{split}
\end{gather}
}
When $M$ is fixed, $\|M' C\|=\|MC\|, \|M'B\|=\|MB\|$ are constant, $\|AC\|$ and $\|AB\|$ are also constant. Since only the rotated angle of $M$ around $BC$ is under change, the only varying entity in (\ref{equivelent2}) is $\|M'A\|$. Considering $t =\|M'A\|^2$ as the variable, then (\ref{equivelent2}) is a quadratic equation in $t$. The coefficient $C_2$of $t^2$ is：
\begin{equation*}
C_2=\|M'B\|^2+\|M'C\|^2-2\|M'B\|\|M'C\|cos\angle A
\end{equation*}
Since $cos\angle A\neq\pm1,C_2>0$, then (\ref{equivelent2}) is an upward parabola. In addition, since $t$ increases monotonically when rotating $M$ to $M'$ around $BC$ by $\theta_M' $for  $\theta_M' \in [0,\pi]$, the maximum of (\ref{equivelent2}) must occur at $\theta_M' =0 $ or  $\theta_M' =\pi$.However since at both $\theta_M' =0 $ or  $\theta_M' =\pi $ and (\ref{equivelent2}) is equal to zero, hence for any $\theta_M' \in [0,\pi]$, (\ref{equivelent2}) always holds. In other words, at a small neighborhood of point $A$, $Cone_{A,AC,(\pi-\beta)}$ and $Cone_{A,AB,(\pi-\gamma)}$ always intersect at two rays from point $A$.
\end{proof}
\begin{figure}[t]
\begin{center}
\includegraphics[width=0.3\linewidth]{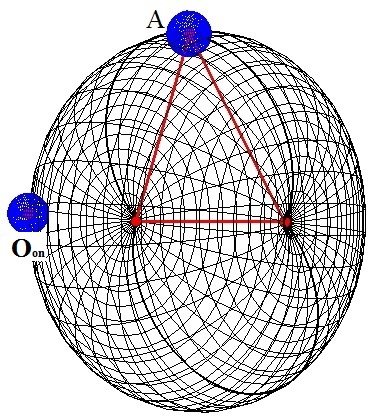}
\end{center}
   \caption{When the optical center $O$ lies on one of six special Toroids, the relationship between solution on the corresponding control point and $O$.}
\label{fig11}
\end{figure}
\textbf{Proof of Theorem 3}
\begin{proof}
 As shown in Fig.\ref{fig11}, $O_{on}$ is an arbitrary point on the toroid $T_{\angle A}$, $O_{out}$ and $O_{in}$ are two points within a small neighborhood of $O_{on}$: $\delta_{O_{on}}=\{O|\|O-O_{on}\|\leq \epsilon_{on} \}$, with $O_{out}$ being outside of $T_{\angle A}$ and $O_{in}$ inside of $T_{\angle A}$. In the next we show that if $\epsilon_{on} $ is sufficiently small, then there always exists a pair of optical center locations, say ($O_{out}',O_{in}'$) in a small neighborhood of point $A$, such that:
Either $\{O_{out}',(A,B,C)\}$ is a solution of the P3P problem$ \{O_{out},(A,B,C)\}$ and $\{O_{in}',(A,B,C)\}$ a S-solution of the P3P problem $\{O_{in},(A,B,C)\}$;
Or $\{O_{out}',(A,B,C)\}$ is a S-solution of the P3P problem $\{O_{out},(A,B,C)\}$ and $\{O_{in}',(A,B,C)\}$ a solution of the P3P problem $\{O_{in},(A,B,C)\}$.
Hence if $\epsilon_1$ is very small, or the optical center $O$ passes through $T_{\angle A}$ from $O_{out}$ to $O_{in}$ ( or from $O_{in}$ to $O_{out}$) , the solution number of the corresponding P3P problem $\{O,(A,B,C)\}$ must change exactly by 1, either increased by 1 or decreased by 1.
Since $O_{on}$ is on $T_{\angle A}$, $O_{out}$ is outside of $T_{\angle A}$, and $O_{in}$ inside of $T_{\angle A}$, we have $\angle BO_{on}C=\angle A $, $\angle BO_{out}C=\angle A_-<\angle A$ and $\angle BO_{in}C=\angle A_+>\angle A$. Then $O_{out}$ is on the toroid $T_{\angle A_-}$, $O_{in}$ is on the toroid $T_{\angle A_+}$.
Suppose $\angle AO_{on}B=\gamma_{on}$, $\angle AO_{on}C=\beta_{on}$, $\angle AO_{out}B=\gamma_{out}$, $\angle AO_{out}C=\beta_{out}$, $\angle AO_{in}B=\gamma_{in}$, $\angle AO_{in}C=\beta_{in}$, since
$\cos \beta=\frac{s_1^2+s_3^2-b^2}{2s_1s_3}$, $\cos \gamma=\frac{s_1^2+s_2^2-c^2}{2s_1s_2}$, and $s_1,s_2,s_3$ are continuously differentiable with respect to the three coordinates of the optical center, if $\epsilon_{on}$ is sufficiently small, $\triangle \beta_{out}=|\beta_{out}-\beta_{on}|$,$\triangle \beta_{in}=|\beta_{in}-\beta_{on}|$,$\triangle \gamma_{out}=|\gamma_{out}-\gamma_{on}|$, and
$\triangle \gamma_{in}=|\gamma_{in}-\gamma_{on}|$ should also be very small.  In other words, $\forall \epsilon>0, \exists \epsilon_{on}>0$, if $O_{out},O_{in}\in \delta_{O_{on}}=\{O|\|O-O_{on} \|\leq \epsilon_{on}\}$, $\triangle \beta_{out},\triangle \beta_{in},\triangle \gamma_{out}$ and $\triangle \gamma_{in} \in (0,\epsilon)$.
By Lemma \ref{Lemma4}, at a small neighborhood $\delta_A$ of point $A$, $\delta_A=\{X|\|X-A\|\leq \epsilon_A \}$,$T_{\beta_{on}}$ can approach infinitesimally to the cone  $Cone_{A,AC,\pi-\beta_{on}}$, $T_{\gamma_{on}}$ to cone  $Cone_{A,AB,\pi-\gamma_{on}}$, and by Lemma \ref{Lemma7}, the intersecting line $L_{\beta_{on},\gamma_{on}}$ of the two toroids $T_{\beta_{on}}$ and $T_{\gamma_{on}}$ can approach infinitesimally to the two intersecting rays of these two cones, $d_{L_{\beta_{on},\gamma_{on}(u)}}$ above the plane $\Pi_{ABC}$ (the control point plane), and $d_{L_{\beta_{on},\gamma_{on}(d)}}$ below  $\Pi_{ABC}$. Similarly for $T_{\beta_{out}}$ and $T_{\gamma_{out}}$, we have their two intersecting rays $d_{L_{\beta_{out},\gamma_{out}(u)}}$ and $d_{L_{\beta_{out},\gamma_{out}(d)}}$, $d_{L_{\beta_{in},\gamma_{in}(u)}}$ and $d_{L_{\beta_{in},\gamma_{in}(d)}}$ for $T_{\beta_{in}}$ and $T_{\gamma_{in}}$.
As shown in the above, since $\triangle \beta_{out},\triangle \beta_{in},\triangle \gamma_{out}$ and $\triangle \gamma_{in} \in (0,\epsilon)$, $d_{L_{\beta_{out},\gamma_{out}(u)}}$ and $d_{L_{\beta_{in},\gamma_{in}(u)}}$ should be within the cone with point $A$ being its apex, $d_{L_{\beta_{on},\gamma_{on}(u)}}$ its axis, and $\Delta \theta_\epsilon $ its half angle. Since
$\lim \limits_{\epsilon \to 0 } \Delta \theta_\epsilon =0$,  $d_{L_{\beta_{out},\gamma_{out}(u)}}$ and $d_{L_{\beta_{in},\gamma_{in}(u)}}$  can approach infinitesimally to each other. Denote the included angle of $d_{L_{\beta_{on},\gamma_{on}(u)}}$  and $n_{T_{\angle A}}$ by $\theta_{n_{T_{\angle A}},L_{\beta_{on},\gamma_{on}(u)}}$.
In the next, we show that depending on whether $ \theta_{n_{T_{\angle A}},L_{\beta_{on},\gamma_{on} (u)} }<\frac{\pi}{2}$ or $ \theta_{n_{T_{\angle A}},L_{\beta_{on},\gamma_{on} (u)} }>\frac{\pi}{2}$, there always exists a pair of optical center locations ($O_{out}',O_{in}'$) in a small neighborhood of point $A$ such that either $\{O_{out}',(A,B,C)\}$ is a solution of the P3P problem $\{O_{out},(A,B,C)\}$ and $\{O_{in}',(A,B,C)\}$ a S-solution of the P3P problem $\{O_{in},(A,B,C)\}$, or $\{O_{out}',(A,B,C)\}$ is a S-solution of the P3P problem $\{O_{out},(A,B,C)\}$ and $\{O_{in}',(A,B,C)\}$ a solution of the P3P problem $\{O_{in},(A,B,C)\}$.\\
(1) When  $ \theta_{n_{T_{\angle A}},L_{\beta_{on},\gamma_{on} (u)} }<\frac{\pi}{2}$\\
If $\epsilon_{on} $ is small enough, then for each point $O_s \in \delta_A=\{X|\|X-A\|\leq \epsilon_A \}$, the included angle $ \theta_{n_{T_{\angle A}},L_{\beta_s,\gamma_s (u)} }$ between its corresponding ray $d_{L_{\beta_s,\gamma_s(u)}}$ and $n_{T_{\angle A}}$ must satisfy:
$ \theta_{n_{T_{\angle A}},L_{\beta_s,\gamma_s (u)} }< \theta_{n_{T_{\angle A}},L_{\beta_{on},\gamma_{on} (u)} }+ \Delta \theta_{\epsilon_{on}} = \theta_{\epsilon_{on}}<\frac{\pi}{2}$.
By Lemma \ref{Lemma6}, if $\epsilon_{on}$ and $\epsilon_A$ are small enough, then the included angle, $\theta_{n_{T_{\angle A_-}}}$, between $\theta_{n_{T_{\angle A}}}$ and the segment connecting point $A$ with an arbitrary intersecting point between $T_{\angle A_-}$ and the outer border of $\delta_A=\{X|\|X-A\|\leq \epsilon_A \}$ must satisfy: $\theta_{\epsilon_{on}}<\theta_{n_{T_{\angle A_-}}}<\frac{\pi}{2}$, which indicates that $L_{\beta_{out},\gamma_{out}}$ must intersect with $T_{\angle A_-}$ at a point $O_{out}'$ in $\delta_A$, then $\{O_{out}',(A,B,C)\}$ is a solution of the P3P problem $\{O_{out},(A,B,C)\}$. At the same time, $L_{\beta_{in},\gamma_{in}}$ cannot intersect with $T_{\angle A_+}$ in $ \delta_A$ .
Denote the intersecting curve between $T_{\pi-\beta_{out}}$ and $T_{\pi-\gamma_{out}}$ by $L_{\pi-\beta_{out},\pi-\gamma_{out}}$, then $L_{\pi-\beta_{out},\pi-\gamma_{out}}$ can be infinitesimally approximated by $d_{L_{\pi-\beta_{out},\pi-\gamma_{out}(u)}}$ which is the opposite ray of $d_{L_{\beta_{out},\gamma_{out}(u)}}$. Similarly, $L_{\pi-\beta_{in},\pi-\gamma_{in}}$ can be infinitesimally approximated by $d_{L_{\pi-\beta_{in},\pi-\gamma_{in}(u)}}$ which is the opposite ray of  $d_{L_{\beta_{in},\gamma_{in}(u)}}$. By the same reasoning as in the above, we can prove in this case that $L_{\pi-\beta_{out},\pi-\gamma_{out}}$ cannot intersect with $T_{\angle A_-}$ in $\delta_A$,  but $L_{\pi-\beta_{in},\pi-\gamma_{in}}$ must intersect with $T_{\angle A_+}$ at a point $O_{in}'$ in $\delta_A$, and $\{O_{in}',(A,B,C)\}$ is a S-solution of the P3P problem $\{O_{in},(A,B,C)\}$.
In sum, if $ \theta_{n_{T_{\angle A}},L_{\beta_{on},\gamma_{on} (u)} }<\frac{\pi}{2}$, then when the optical center $O$ passes through the toroid $T_{\angle A}$ from $O_{out}$ to $O_{in}$, either a solution becomes a S-solution, or a S-solution becomes a solution for the corresponding P3P problem, hence the number of the solutions must change by 1.\\
(2) When $ \theta_{n_{T_{\angle A}},L_{\beta_{on},\gamma_{on} (u)} }>\frac{\pi}{2}$\\
The conclusion can be similarly obtained.\\
(3)  When $ \theta_{n_{T_{\angle A}},L_{\beta_{on},\gamma_{on} (u)} }=\frac{\pi}{2}$\\
We can prove that if $\angle B <\frac{\pi}{2},$ and $\angle C <\frac{\pi}{2}$, $ \theta_{n_{T_{\angle A} },L_{\beta_{on},\gamma_{on} (u)} }<\frac{\pi}{2}$ always holds. Otherwise $ \theta_{n_{T_{\angle A} },L_{\beta_{on},\gamma_{on} (u)} }=\frac{\pi}{2}$ implies that the optical center must lie on two concentric coplanar circles with the orthogonal line AB passing through their common center. In this case, the solution changes depend on also on the specific passing direction, merely specifying from inside to outside is not sufficient.
Combining the above (1), (2) and (3), Theorem 3 is proved.
\end{proof}

Before ending this section, we have additionally the following two results:
\begin{thm}\label{thm4}
 When the optical center $O$ passes through the outer surface of $\overline{T_{union}}$ from outside to inside, the number of the solutions of the corresponding P3P problem must decrease by 1.
\end{thm}
\begin{thm}\label{thm5}
When the optical center $O$ passes through toroid  $T_{\pi-\angle A}(or T_{\pi-\angle B},or T_{\pi-\angle C})$  except possibly for two concentric coplanar circles on it, the number of the solutions of the corresponding P3P problem does not change, but its S-solution changes from one to another.
\end{thm}

\section{Conclusion}
In this work, some new insights of the multi-solution phenomenon of the P3P problem are obtained. We show that given 3 control points, if the optical center is outside of the 6 toroids defined by the 3 control points, the corresponding P3P problem cannot have a unique solution, and  a one-to-one root-solution correspondence does exist. In addition we show that three of the 6 toroids must act as some of the critical surfaces, or when the optical center passes through any one of these three toroids except possibly for some special circular curves, the solution number of the corresponding P3P problem must change exactly by one. In addition if the toroid is part of the outer surface of the union of the 6 toroids, the number of the solutions must decrease by one if the passing of the optical center is from the outside to inside, and increase by one if the passing is in the opposite direction. These new findings, in addition to their academic values, could also act as some theoretical guidance for P3P practitioners.\\
{\small
\bibliographystyle{ieee}
\bibliography{egbibsecond}
}
\end{document}